\documentclass{article}

\usepackage{microtype}
\usepackage{graphicx}
\usepackage{subfigure}
\usepackage{booktabs} 

\usepackage{hyperref}

\usepackage[accepted]{icml2019}

\usepackage{url}
\usepackage{graphicx}
\usepackage{booktabs} 
\usepackage{amsmath, bm}
\usepackage{color}
\usepackage{multirow}
\usepackage{amssymb}
\usepackage{footnote}
\usepackage{epstopdf}
\makesavenoteenv{table}
\usepackage{xspace}
\usepackage{amsthm}

\theoremstyle{plain}
\newtheorem{theorem}{Theorem}

\newtheorem{proposition}[theorem]{Proposition}

\newtheorem{corollary}[theorem]{Corollary}

\theoremstyle{remark}
\newtheorem{remark}[theorem]{Remark}

\graphicspath{{figures/}}

\makeatletter
\def\BState{\State\hskip-\ALG@thistlm}
\makeatother

\begin{document}
	\twocolumn[
	\icmltitle{RaFM: Rank-Aware Factorization Machines}
	
	
	
	\icmlsetsymbol{equal}{*}
	
	\begin{icmlauthorlist}
		\icmlauthor{Xiaoshuang Chen}{tsinghua}
		\icmlauthor{Yin Zheng}{tencent}
		\icmlauthor{Jiaxing Wang}{ucas}
		\icmlauthor{Wenye Ma}{tencent}
		\icmlauthor{Junzhou Huang}{tencent}
	\end{icmlauthorlist}
	
	\icmlaffiliation{tsinghua}{Department of Electrical Engineering, Tsinghua University, Beijing, China}
	\icmlaffiliation{tencent}{Tencent AI Lab, Shenzhen, China}
	\icmlaffiliation{ucas}{Institute of Automation, Chinese Academy of Sciences, and University of Chinese Academy of Sciences, Beijing, China}
	
	\icmlcorrespondingauthor{Yin Zheng}{yinzheng@tencent.com}
	
	\icmlkeywords{Factorization Machines}
	
	\vskip 0.3in
	]
	
	\printAffiliationsAndNotice{}
	
	\begin{abstract}
		Fatorization machines (FM) are a popular model class to learn pairwise interactions by a low-rank approximation. Different from existing FM-based approaches which use a fixed rank for all features, this paper proposes a Rank-Aware FM (RaFM) model which adopts pairwise interactions from embeddings with different ranks. The proposed model achieves a better performance on real-world datasets where different features have significantly varying frequencies of occurrences. Moreover, we prove that the RaFM model can be stored, evaluated, and trained as efficiently as one single FM, and under some reasonable conditions it can be even significantly more efficient than FM. RaFM improves the performance of FMs in both regression tasks and classification tasks while incurring less computational burden, therefore also has attractive potential in industrial applications.
	\end{abstract}

	\section{Introduction}\label{section:introduction}
	Factorization machines (FM)~\cite{rendle2010factorization,rendle2012factorization} are one of the 
	most popular models to leverage the interactions between features. It models nested feature interactions 
	via a factorized parametrization, and achieves success in many sparse predictive areas, such as recommender systems and click-through rate predictions \cite{juan2016field}. Recently there are many follow-up researches, such as high-order FMs \cite{blondel2016higher}, convex FMs \cite{blondel2015convex}, neural-network-based FMs \cite{he2017neural,guo2017deepfm}, locally linear FMs \cite{liu2017locally}, etc.
	
	Most of the existing approaches allocate a fixed rank of embedding vectors to each feature~\cite{liu2017locally,guo2017deepfm,Zheng2016,zheng2016neural,du2018collaborative,lauly2017document,jiang2017}. However, the frequencies of different features occurring in real-world datasets vary a lot. Fig. \ref{fig:feature-freq} shows the numbers of occurrences of different features in two public datasets, i.e. MovieLens\footnote{\url{https://grouplens.org/datasets/movielens}} Tag and Avazu\footnote{\url{https://www.kaggle.com/c/avazu-ctr-prediction/data}} respectively. It is shown that a large number of features scarsely occur in the dataset, while only a few features frequently occur. Fixing the rank of FMs may lead to overfitting problems for features with few occurrences, and underfitting problems for features with many occurrences.
	
	\begin{figure}
		\centering
		\includegraphics[width=0.8\linewidth]{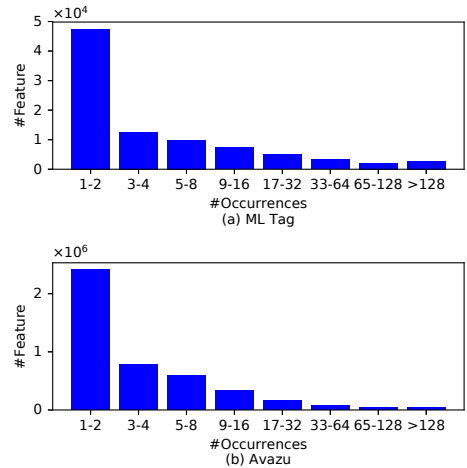}
				\vspace{-6mm}
		\caption{Occurrences of different features in (a) ML Tag; (b) Avazu. Example: in ML Tag, more than 40,000 features only occur once or twice, while less than 5,000 features occur more than 128 times.}
		\label{fig:feature-freq}
				\vspace{-5mm}
	\end{figure}
	
	
	
	There are also researches allocating with different ranks to each feature\cite{li2016difacto,li2017mixture,juan2016field}. However, training and storing multiple embeddings lead to unaffordable computational burden, and the large number of parameters may even aggravate the overfitting problem \cite{li2017mixture}. To this end, this paper proposes a Rank-Aware FM (RaFM) model, i.e., to maintain embedding vectors with different ranks for each feature, which makes it possible to compute each pairwise interaction via proper ranks of embedding vectors. A key contribution of this paper is to prove that {\bf although RaFM maintains multiple embeddings with different ranks for each feature, it can be stored, evaluated, and trained as efficiently as, or even more efficiently than a single FM with a fixed rank}.

	Specifically, we analyze the time and space complexity of RaFM, and show that there are many inactive factors which need not be stored and evaluated. Therefore, both the computational time and the storage can be significantly reduced, and the entire computational burden can be even smaller than that of a single FM under some reasonable conditions. Furthermore,  we provide an algorithm to train all embedding vectors of RaFM in one concise model where the inactive factors will not be stored and trained, and then prove the convergence rate and the performance bound of the training algorithm. Thus, the computational burden of the training process can also be effectively reduced. Experiments show the effectiveness of RaFM in both public datasets and datasets of industrial applications. The proposed RaFM model improves the performance of FMs while incurring a comparable or even less computational burden, therefore also has attractive potential in industrial applications.
	
	
		\vspace{-1mm}
	\section{Problem Formulation} \label{section:efm}
		\vspace{-1mm}
	\subsection{General Interaction Form of FMs}
	Here we provide a general interaction form of the FM model:
	\begin{equation} \label{eq:g-fm}
	\hat{y} = \sum_{i,j\in\mathcal{F},i<j}\left<\mathcal{V}_i,\mathcal{V}_j\right>x_ix_j + \sum_iw_ix_i + bias
	\end{equation}
	where $\mathcal{F}$ is the index set of features, $x_i$ is the value of the $i$-th feature, $w_i$ is the weight of the $i$-th feature in the linear part, and $bias$ is the bias of the model. $\mathcal{V}_i$ is the embedding of the $i$-th feature, describing the characteristics of this feature when interacting with other features. $\left<\mathcal{V}_i,\mathcal{V}_j\right>$ represents the bi-interaction between the $i$-th feature and the $j$-th feature.
	
	Clearly, the computation of the bi-interaction $\left<\mathcal{V}_i,\mathcal{V}_j\right>$ highly depends on the formulations of the embedding term $\mathcal{V}_i$. In the original FM model, $\mathcal{V}_i$ is a vector of a fixed rank, and $\left<\mathcal{V}_i,\mathcal{V}_j\right>$ is the dot product. Specifically, assume we have $m$ FMs, the rank of the $k$-th of which is $D_k$, then we have
	\begin{equation} \label{eq:fm}
	\left<\mathcal{V}_i,\mathcal{V}_j\right>_{FM_k} = \bm{v}_i^{(k)}\cdot\bm{v}_j^{(k)}=\sum_{f=1}^{D_k}v_{i,f}^{(k)}v_{j,f}^{(k)}
	\end{equation}
	where $\bm{v}_i^{(k)}$ is the embedding with dimension $D_k$, and $v_{i,f}^{(k)}$ is the $f$-th coordinate of $\bm{v}_i^{(k)}$. We assume the $m$ FMs are ordered such that $D_1\leq D_2 \leq \cdots \leq D_m$. Let $D = D_m$. We use $\bm{v}^{(k)}$ to represent the parameter set $\{\bm{v}_i^{(k)}:i\in\mathcal{F}\}$.
	
	According to Fig. \ref{fig:feature-freq}, the frequencies of different features occurring in real-world datasets vary a lot, hence fixed rank FMs may have unsatisfying performance on these datasets. Specifically, a certain embedding $\mathcal{V}_i$ will be trained only when $x_i$ are both nonzero. Therefore in the $k$-th FM, the embedding $\mathcal{V}_i = \bm{v}_i^{(k)}$ will suffer from underfitting if the $i$-th feature frequently occurs, and from overfitting if the feature scarsely occurs, as shown in Fig. \ref{fig:efm-scheme}(a)(b).
	
	\subsection{RaFM: Rank-Aware Factorization Machines}
	From the abovementioned discussions, it is beneficial to allocate different ranks of embedding vectors to different features. 
	In this paper, we allocate multiple embedding vectors to each feature in RaFM, i.e.,
	\begin{equation} \label{eq:efm}
	\mathcal{V}_i = \left\{\bm{v}_i^{(1)},\bm{v}_i^{(2)},\cdots, \bm{v}_i^{(k_i)}\right\}
	\end{equation}
	where the rank of the $k$-th vector is $D_k$. The value $k_i$ means that the maximum rank of the embedding vectors is $D_{k_i}$. $k_i$ can be interpreted as: if $k>k_i$, the rank of $\bm{v}_i^{(k)}$ will be too large compared to the occurrences of the $i$-th feature, which will lead to the overfitting problem. $k_i$ can be chosen according to the number of occurrences of the $i$-th feature, and in this paper, we regard $k_i$ as hyperparameters. Although some boosting
	methods~\cite{li2018boosting} might be used to select appropriate $k_i$, we leave it for future work.
	
	Here we emphasize that the multiple embeddings in $\mathcal{V}_i$ are not independent in the RaFM model. Intuitively speaking, $\bm{v}_i^{(1)}$ can be regarded as a ``projection'' of $\bm{v}_i^{(2)}$ from the $D_2$-dimensional space to the $D_1$-dimensional space. We will return to this point in Section \ref{section:train}, which provides an efficient training problem maintaining this property.
	 
    The pairwise interaction can be computed as
	\begin{equation} \label{eq:assumption}
	\left<\mathcal{V}_i,\mathcal{V}_j\right>_{RaFM} = \bm{v}_i^{(k_{ij})}\cdot\bm{v}_j^{(k_{ij})}
	\end{equation}
	where
	\begin{equation}
	    k_{ij}=\min(k_i,k_j)
	\end{equation}
	which means the dimensionality to compute the pairwise interaction between the $i$-th and $j$-th features depends on the maximum common dimensionality of the embedding vectors of the $i$-th and the $j$-th feature. Therefore, the computation of $\left<\mathcal{V}_i,\mathcal{V}_j\right>$ achieves a good trade-off between overfitting problems and underfitting problems, which is the intuition why RaFM outperforms FM. Fig.~\ref{fig:efm-scheme} also illustrates the idea of RaFM.
	
	\begin{figure}[t]
	    \centering
	    \includegraphics[width=\columnwidth]{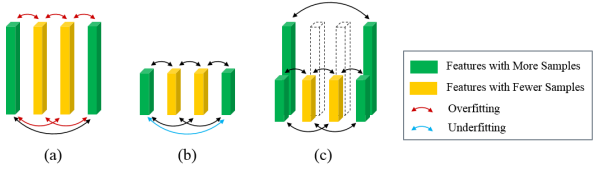}
	    \vspace{-5mm}
	    \caption{The intuition of RaFM. (a) FM with a high dimension; (b) FM with a low dimension; (c) RaFM}
	    \label{fig:efm-scheme}
	    \vspace{-5mm}
	\end{figure}
	
	A key challenge of RaFM is the computational burden when storing, evaluating, and training RaFM:
	\begin{enumerate}
		\item It is well-known that FM can be evaluated efficiently in $O(D\left|\mathcal{F}\right|)$. However, the straight computation of RaFM, i.e. \eqref{eq:g-fm}\eqref{eq:assumption}, is $O(D\left|\mathcal{F}\right|^2)$, which is not satifying.
		\item To train the RaFM, we need to train multiple embedding vectors for a feature, which incurs a large computational burden and storage consumption.
	\end{enumerate}
	
	We separate the discussions into two sections: Section \ref{section:complexity} shows that RaFM can be stored and evaluated as efficient, or even more efficient than a single FM, whereas Section \ref{section:train} provides an efficient training algorithm.
	
	\color{black}
	\vspace{-2mm}
	\section{Model Complexity of RaFM}
	\label{section:complexity}
	This section analyzes the space and time complexity of RaFM. Trivial bounds of the space and time complexity are $O(\sum_{k=1}^mD_k\left|\mathcal{F}\right|)$ and $O(D\left|\mathcal{F}\right|^2)$ respectively. In this section, they will be significantly improved to be comparable or even less than a single FM under reasonable conditions. 
	
	%
	
	
	\vspace{-2mm}
	\subsection{Space Complexity} \label{section:space-complexity}
	Before discussing the complexity of RaFM, we introduce a class of feature sets, i.e.
	\begin{equation}
	\mathcal{F}_k = \{i\in\mathcal{F}:k_i\geq k\}, k=1,2,\cdots,m
	\end{equation}
	then we have $\mathcal{F}_1 \supset \mathcal{F}_2 \cdots \supset \mathcal{F}_m$. We have $\mathcal{F}=\mathcal{F}_1$ due to $k_i\geq 1$ by definition.
	
	Although the RaFM in Eq. \eqref{eq:efm} contains $m$ group of embeddings with different ranks, not all the parameters each group of embeddings will be used. Specifically, if we use $\mathcal{F}-\mathcal{F}_k$ to denote the set difference of $\mathcal{F}$ and $\mathcal{F}_k$, then for $i\in\mathcal{F}-\mathcal{F}_{k}$ and $l>k$, the factor $\bm{v}_i^{(l)}$ will never be used according to Eq. \eqref{eq:assumption}. These factors are called \textit{inactive factors}. In other words, only $\bm{v}^{(k)}_i,i\in\mathcal{F}_k$ need to be maintained, which are called \textit{active factors}. Therefore we have
	\begin{proposition} \label{prop:space-complexity}
		The space complexity of parameters in RaFM \eqref{eq:efm} is $O \left( \sum_{k=1}^mD_k\left|\mathcal{F}_k\right| \right) $.
	\end{proposition}
	
	\vspace{-2mm}
	\subsection{Time Complexity} \label{section:time-complexity}
	To derive the time complexity of RaFM, we introduce two notations, i.e. $\mathcal{A}_{l,k}$ and $\mathcal{B}_{l,k}$. Specifically,
	\begin{equation} \label{eq:a-l-k}
	\begin{split}
	\mathcal{A}_{l,k} &= \sum_{i,j\in\mathcal{F}_k,i< j}\bm{v}_i^{(l)}\cdot\bm{v}_j^{(l)}x_ix_j \\ &= \frac{1}{2}\left(\left\|\sum_{i\in\mathcal{F}_k}\bm{v}_i^{(l)}x_i\right\|_2^2-\sum_{i\in\mathcal{F}_k}\left\|\bm{v}_i^{(l)}x_i\right\|_2^2\right)
	\end{split}
	\end{equation}
	And when $l\leq k$, $\mathcal{B}_{l,k}$ is defined as
	\begin{equation} 
	\mathcal{B}_{l,k} = \sum_{i< j}\bm{v}_i^{(k_{ij}|_{[l,k]})}\cdot\bm{v}_j^{(k_{ij}|_{[l,k]})}x_ix_j
	\end{equation}
	where $k_{ij}|_{[l,k]}=\max[l,\min(k,k_{ij})]$.
	
	$\mathcal{A}_{l,k}$ and $\mathcal{B}_{l,k}$ are introduced in order to transform the original RaFM in Eq.~\eqref{eq:g-fm}\eqref{eq:assumption} to a computationally efficient form. Actually, the following properties are obvious:
	\begin{proposition} \label{prop:a-b-simple}
	We have the following properties regarding $\mathcal{A}_{l,k}$ and $\mathcal{B}_{l,k}$:
	\begin{enumerate}
		\item $\mathcal{A}_{k,1}$ and $\mathcal{B}_{k,k}$ both denote the $k$-th FM model in Eq.\eqref{eq:fm}.
		\item The RaFM model in Eq. \eqref{eq:g-fm}\eqref{eq:assumption} is $\mathcal{B}_{1,m}$;
		\item The computational complexity of $\mathcal{A}_{l,k}$ is $O(D_l\left|\mathcal{F}_k\right|)$ due to the second equality in Eq. \eqref{eq:a-l-k}.
	\end{enumerate}
	\end{proposition}
	The proof is omitted since these statements are obvious. According to Proposition \ref{prop:a-b-simple}, the time complexity of $\mathcal{A}_{l,k}$ is known, and all we need to do is to compute $\mathcal{B}_{l,k}$ according to $\mathcal{A}$. To achieve this, we provide the following theorem:
	\begin{theorem} \label{thm:a-b}
	$\mathcal{A}_{l,k}$ and $\mathcal{B}_{l,k}$ satisfy the following equality: $\mathcal{B}_{l,k+1} = \mathcal{B}_{l,k}-\mathcal{A}_{k,k+1}+\mathcal{A}_{k+1,k+1}$
	\end{theorem}
	\begin{proof}
	It is easy to show that
		\begin{equation}
		\begin{split}
		\mathcal{B}_{l,k+1}=\mathcal{B}_{l,k}&- \sum_{i< j,k_{ij}> k}\bm{v}_i^{(k)}\cdot\bm{v}_j^{(k)}x_ix_j\\
		&+\sum_{i< j,k_{ij}> k}\bm{v}_i^{(k+1)}\cdot\bm{v}_j^{(k+1)}x_ix_j
		\end{split}
		\end{equation}
		Moreover, according to \eqref{eq:assumption}, the feature set $\{i<j:k_{ij}>k\}$ can be rewritten as
		\begin{equation}
		\{i<j:k_{ij}>k\}=\{i<j:i,j\in\mathcal{F}_{k+1}\}
		\end{equation}
		Then we have the 4th property according to the definition of $\mathcal{A}_{k,k+1}$ and $\mathcal{A}_{k+1,k+1}$.
	\end{proof}
	
	\begin{corollary} \label{corollary:time}
	    Regarding the computational complexity of RaFM, we have
		\begin{enumerate}
		\item The time complexity of $\mathcal{B}_{l,k}$ is $O \left(D_l\left|\mathcal{F} \right|+ \sum_{p=l+1}^kD_p\left|\mathcal{F}_p\right| \right) $;
		\item The time complexity of the RaFM model \eqref{eq:efm}, i.e. $\mathcal{B}_{1,m}$, is $O \left( \sum_{k=1}^mD_k\left|\mathcal{F}_k\right| \right) $.
		\end{enumerate}
	\end{corollary}
	\begin{proof}
		We only prove the 1st statement, and the 2nd statement is the direct corollary of the 1st when $l=1,k=m$.
		
		We prove by induction. When $k=l$, the time complexity of $\mathcal{B}_{l,l}$ is $O\left(D_l\left|\mathcal{F}\right|\right)$. If the time complexity of $\mathcal{B}_{l,k}$ is $O \left(D_l\left|\mathcal{F} \right|+ \sum_{p=l+1}^kD_p\left|\mathcal{F}_p\right| \right) $, then the time complexity of $\mathcal{B}_{l,k+1}$ should be
		\begin{equation}
		\begin{split}
		&O \left(D_l\left|\mathcal{F} \right|+ \sum_{p=l+1}^kD_p\left|\mathcal{F}_p\right| \right)  + O\left(D_k\left|\mathcal{F}_{k+1}\right|\right) \\&+ O\left(D_{k+1}\left|\mathcal{F}_{k+1}\right|\right) = O \left(D_l\left|\mathcal{F} \right|+ \sum_{p=l+1}^{k+1}D_p\left|\mathcal{F}_p\right| \right)
		\end{split}
		\end{equation}
		where the fact $D_k\leq D_{k+1}$ is used.
	\end{proof}
	
	\begin{remark}
		Similar to FM, when the data is sparse, $\left|\mathcal{F}_k\right|$ should be replaced by $n(\mathcal{F}_k)$, which is the expected number of occurrences of $\mathcal{F}_k$ in a data sample. In such cases, the time complexity is in the sense of expectation.
	\end{remark} 
	
	\subsection{Comparison with FM} \label{section:complexity-discussion}
	This section compares the complexity of RaFM with that of FM. When using a single FM as the predictor, we usually use a large rank to ensure the performance, and use regularization to avoid overfitting. Here we use the FM with rank $D_m$ for comparison, of which the space and time complexities are $O(D\left|\mathcal{F}\right|)$ and $O(Dn(\mathcal{F}))$ respectively.
	
	The space and time complexity of RaFM are similar to each other except that $\left|\mathcal{F}_k\right|$ should be replaced by $n(\mathcal{F}_k)$ in the time complexity, so we discuss them together. Generally, when $k$ increases, $D_k$ will increase and $\left|\mathcal{F}_k\right|$ will decrease. Table \ref{tab:complexity} provides the complexity under different speeds of $D_k$ increasing and $\left|\mathcal{F}_k\right|$ decreasing. It is shown that if $D_k$ increases and $\left|\mathcal{F}_k\right|$ decreases moderately (e.g. linearly), the complexity of RaFM will be large. However, if one or two of them vary rapidly (e.g. exponentially), the complexity of RaFM will be comparable or smaller than the FM model.
	
	\begin{table}
		\centering
		\begin{small}
			\caption{Complexity of RaFM under different conditions}
			\resizebox{\linewidth}{\height}{
				\begin{tabular}{ccc}
					\hline
					&$D_k=\Theta(\frac{k}{m}D)$&$D_k=\Theta(2^{k-m}D)$ \\
					\hline
					$\left|\mathcal{F}_k\right|=\Theta((1-\frac{k-1}{m})\left|\mathcal{F}\right|)$&$O(mD\left|\mathcal{F}\right|)$&$O(D\left|\mathcal{F}\right|)$\\
					\hline
					$\left|\mathcal{F}_k\right|=\Theta(2^{1-k}\left|\mathcal{F}\right|)$&$O(D\left|\mathcal{F}\right|)$&$O(\frac{m}{2^{m-1}}D\left|\mathcal{F}\right|)$\\
					\hline
				\end{tabular}
			}
			\label{tab:complexity}
		\end{small}
	\end{table}
	
	In practice, it is widely-accepted that $D_k$ varies rapidly when $k$ increases \cite{he2017neural,li2017mixture}. As indicated by Fig. \ref{fig:feature-freq}, $\left|\mathcal{F}_k\right|$ decreases rapidly when $k$ increases. In contrast, the speed of $n(\mathcal{F}_k)$ decreasing may not be as rapidly as that of $\left|\mathcal{F}_k\right|$, but is still likely to be superlinear. Therefore, in such conditions, \textit{RaFM will significantly reduce the space complexity of FM while incur a comparable or smaller time complexity.} This statement will also be validated by experiments in Section \ref{section:experiment-complexity}.
	
		\vspace{-2mm}
	\section{Efficient Learning of RaFM} \label{section:train}
	
	\begin{algorithm}
		\caption{Training the RaFM} \label{algo:learning}
		\begin{algorithmic}[1]
			\STATE Initialize all the parameters
			
			\WHILE{not convergent}
			
			\STATE Sample a data point $(\bm{x},y)$ randomly
			
			\FOR{$1\leq p<m$}
			\STATE \hspace{-4mm}$\left.\bm{v}^{(p)}\right|_{\mathcal{F}_{p+1}} \gets \left.\bm{v}^{(p)}\right|_{\mathcal{F}_{p+1}} - \rho_d\frac{\partial L(\mathcal{B}_{1,p},\mathcal{B}_{1,p+1})}{\partial \left.\bm{v}^{(p)}\right|_{\mathcal{F}_{p+1}}}$
			
			\STATE \hspace{-4mm}$\left.\bm{v}^{(p)}\right|_{\mathcal{F}_p-\mathcal{F}_{p+1}} \gets \left.\bm{v}^{(p)}\right|_{\mathcal{F}_p-\mathcal{F}_{p+1}} - \rho_f\frac{\partial L(\mathcal{B}_{1,m}, y)}{\partial \left.\bm{v}^{(p)}\right|_{\mathcal{F}_p-\mathcal{F}_{p+1}}}$
			\ENDFOR
			\STATE $\left.\bm{v}^{(m)}\right|_{\mathcal{F}_m} \gets \left.\bm{v}^{(m)}\right|_{\mathcal{F}_m} - \rho_f\frac{\partial L(\mathcal{B}_{1,m}, y)}{\partial \left.\bm{v}^{(m)}\right|_{\mathcal{F}_m}}$
			
			\ENDWHILE
		\end{algorithmic}
				\vspace{-1mm}
	\end{algorithm}
	
	This section provides a computationally efficient learning algorithm of RaFM, i.e. Algorithm \ref{algo:learning}, where the inactive factors need not be stored and trained. We first provide the objective function and the learning algorithm, then prove that the proposed algorithm is to train the upper bounds of all the $m$ FMs simultaneously.
	
	We emphasize that although the analysis in this section is somehow technical, the final algorithm, i.e., Algorithm \ref{algo:learning}, can be regarded as an extension to SGD methods, hence is easy to implement.
	
	\subsection{Constrained Optimization of RaFM}\label{section:train-obj}
	The goal of the training algorithm is to obtain the parameters in $\mathcal{B}_{1,m}$. According to Eq.~\eqref{eq:a-l-k} and Theorem ~\ref{thm:a-b}, the parameters are $\left.\bm{v}^{(p)}\right|_{\mathcal{F}_p},\forall 1\leq p\leq m$. In other words, $\left.\bm{v}^{(p)}\right|_{\mathcal{F}-\mathcal{F}_p},\forall 1\leq p\leq m$ are the inactive factors that need not be trained. In order to avoid the training of inactive factors, we provide the following bi-level optimization model:
	\begin{subequations} \label{eq:efm-train-opt}
		\begin{align} 
		&\min \frac{1}{N}\sum_{\bm{x}}L(\mathcal{B}_{1,m}, y) \label{eq:efm-train-opt-obj}\\
		\text{s.t.~} &\left.\bm{v}^{(p)}\right|_{ \mathcal{F}_{p+1}} = \arg\min \frac{1}{N}\sum_{\bm{x}}L(\mathcal{B}_{1,p},\mathcal{B}_{1,p+1}),\forall 1\leq p < m \label{eq:efm-train-opt-constraint}
		\end{align}
	\end{subequations}	
	where $L$ is the loss function, and $\left.\bm{v}^{(p)}\right|_{ \mathcal{F}_{p+1}}$ denotes the $\bm{v}_i^{(p)}$ where $i\in\mathcal{F}_{p+1}$.
	
	
	The variables in Eq. \eqref{eq:efm-train-opt} can be classified into two groups, i.e. free variables and dependent variables.  $\left.\bm{v}^{(p)}\right|_{\mathcal{F}_p - \mathcal{F}_{p+1}},\forall 1\leq p< m$ and $\left.\bm{v}^{(m)}\right|_{ \mathcal{F}_m}$ are free variables, while $\left.\bm{v}^{(p)}\right|_{ \mathcal{F}_{p+1}}, \forall 1\leq p<m$ are dependent variables since they are determined by $\left.\bm{v}^{(p+1)}\right|_{ \mathcal{F}_{p+1}}$ via the constraint \eqref{eq:efm-train-opt-constraint}. The basic idea of Eq. \eqref{eq:efm-train-opt} is to regard $\bm{v}_i$ in its highest dimensionality as free variables to be optimized, and to regard other lower-dimensionality counterparts as its ``projections'' in lower dimensions, which can be approximated by the constraint \eqref{eq:efm-train-opt-constraint}.
	Note that inactive factors $\left.\bm{v}^{(p)}\right|_{\mathcal{F}-\mathcal{F}_p},1\leq p \leq m$ does not exist in Eq. \eqref{eq:efm-train-opt}. Therefore, Eq. \eqref{eq:efm-train-opt} makes it possible to maintain the model size given by Proposition \ref{prop:space-complexity} in the training process.
	
		\vspace{-2mm}
	\subsection{Learning} \label{section:train-sgd}
	This subsection will show that Eq. \eqref{eq:efm-train-opt} can be efficiently trained. Due to \eqref{eq:efm-train-opt-constraint}, the dependent variables $\left.\bm{v}^{(p)}\right|_{\mathcal{F}_{p+1}}$ can be obtained by taking the stochastic gradient descent (SGD) algorithm on the first term in $L(\mathcal{B}_{1,p},\mathcal{B}_{1,p+1})$. The major challenge is the estimation of the gradient direction of the free variables due to that the optimization problem \eqref{eq:efm-train-opt} is a multi-stage optimization problem. Taking $\left.\bm{v}^{(p)}\right|_{\mathcal{F}_p-\mathcal{F}_{p+1}}$ for a certain $p$ as an example, the gradient with respect to $\left.\bm{v}^{(p)}\right|_{\mathcal{F}_p-\mathcal{F}_{p+1}}$ should be
	\begin{equation} \label{eq:grad}
	\begin{split}
	grad= \frac{1}{N}\sum_{\bm{x}}&\left(\frac{\partial L(\mathcal{B}_{1,m},y)}{\partial \left.\bm{v}^{(p)}\right|_{\mathcal{F}_p-\mathcal{F}_{p+1}}} \right. \\
	&\left.+ \frac{\partial L(\mathcal{B}_{1,m},y)}{\partial \left.\bm{v}^{(p-1)}\right|_{\mathcal{F}_p}} \frac{\partial \left.\bm{v}^{(p-1)}\right|_{\mathcal{F}_p}}{\partial \left.\bm{v}^{(p)}\right|_{\mathcal{F}_p-\mathcal{F}_{p+1}}}\right)
	\end{split}
	\end{equation}
	where $\partial \left.\bm{v}^{(p-1)}\right|_{\mathcal{F}_p} / \partial \left.\bm{v}^{(p)}\right|_{\mathcal{F}_p-\mathcal{F}_{p+1}}$ is a $\left|\mathcal{F}_p\right|D_{p-1}\times (\left|\mathcal{F}_{p}\right|-\left|\mathcal{F}_{p+1}\right|)D_p$ matrix representing the relationship between the dependent variable $\left.\bm{v}^{(p-1)}\right|_{\mathcal{F}_p}$ and the free variable $\left.\bm{v}^{(p)}\right|_{\mathcal{F}_p-\mathcal{F}_{p+1}}$ given by the constraint \eqref{eq:efm-train-opt-constraint}. According to the theorem of implicit functions,
	\begin{equation} \label{eq:implicit}
	\begin{split}
	\frac{\partial \left.\bm{v}^{(p-1)}\right|_{\mathcal{F}_p}}{\partial \left.\bm{v}^{(p)}\right|_{\mathcal{F}_p-\mathcal{F}_{p+1}}}=& \left[\frac{1}{N}\sum_{\bm{x}}\frac{\partial^2 L(\mathcal{B}_{1,p-1},\mathcal{B}_{1,p})}{\partial \left(\left.\bm{v}^{(p-1)}\right|_{\mathcal{F}_p}\right)^2}\right]^{-1} \\
	&
	\left[\frac{1}{N}\sum_{\bm{x}}\frac{\partial^2 L(\mathcal{B}_{1,p-1},\mathcal{B}_{1,p})}{\partial \left.\bm{v}^{(p-1)}\right|_{\mathcal{F}_p} \partial \left.\bm{v}^{(p)}\right|_{\mathcal{F}_p-\mathcal{F}_{p+1}}}\right]
	\end{split}
	\end{equation}
	Apply Eq. \eqref{eq:implicit} to Eq. \eqref{eq:grad},
	\begin{equation} \label{eq:grad-derive}
	\begin{split}
	&grad = \frac{1}{N}\sum_{\bm{x}}\frac{\partial L(\mathcal{B}_{1,m},y)}{\partial \left.\bm{v}^{(p)}\right|_{\mathcal{F}_p-\mathcal{F}_{p+1}}}\\ &+\frac{1}{N^2}\sum_{\bm{x,x'}} \frac{\partial L(\mathcal{B}_{1,m},y)}{\partial \left.\bm{v}^{(p-1)}\right|_{\mathcal{F}_p}}
	\bm{G}^{-1}
	\frac{\partial^2 L(\mathcal{B}'_{1,p-1},\mathcal{B}'_{1,p})}{\partial \left.\bm{v}^{(p-1)}\right|_{\mathcal{F}_p} \partial \left.\bm{v}^{(p)}\right|_{\mathcal{F}_p-\mathcal{F}_{p+1}}}
	\end{split}
	\end{equation}
	where $\bm{G} = \frac{1}{N}\sum_{\bm{x}}\left[\partial^2 L/\partial \left(\left.\bm{v}^{(p-1)}\right|_{\mathcal{F}_p}\right)^2\right]$ is the first term in the right of \eqref{eq:implicit}, which is a $\left|\mathcal{F}_p\right|D_{p-1}\times \left|\mathcal{F}_p\right|D_{p-1}$ matrix. $\mathcal{B}'_{l,k}$ denotes the $\mathcal{B}_{l,k}$ of input vector $\bm{x}'$. By exchanging $\bm{x}$ and $\bm{x}'$ in the second term of the right of \eqref{eq:grad-derive}, and using one sample $\bm{x}$ to estimate the gradient, we have
		\vspace{-1mm}
	\begin{equation} \label{eq:grad-sgd}
	\begin{split}
	&\widehat{grad} = \frac{\partial L(\mathcal{B}_{1,m},y)}{\partial \left.\bm{v}^{(p)}\right|_{\mathcal{F}_p-\mathcal{F}_{p+1}}} \\ &+\frac{1}{N}\sum_{\bm{x'}} \frac{\partial L(\mathcal{B}'_{1,m},y)}{\partial \left.\bm{v}^{(p-1)}\right|_{\mathcal{F}_p}}
	\bm{G}^{-1}
	\frac{\partial^2 L(\mathcal{B}_{1,p-1},\mathcal{B}_{1,p})}{\partial \left.\bm{v}^{(p-1)}\right|_{\mathcal{F}_p} \partial \left.\bm{v}^{(p)}\right|_{\mathcal{F}_p-\mathcal{F}_{p+1}}}
	\end{split}
	\end{equation}
	
	\vspace{-1mm}
	In the right of Eq. \eqref{eq:grad-sgd}, the first term can be obtain by the chain rule, while the second term contains second order derivatives which are challenging to compute. However, we have the following theorem:
	\begin{theorem} \label{thm:sgd}
		The direction of $\widehat{grad}$ is parallel to its first term $\partial L(\mathcal{B}_{1,m},y)/\partial \left.\bm{v}^{(p)}\right|_{\mathcal{F}_p-\mathcal{F}_{p+1}}$.
	\end{theorem}
		\vspace{-3mm}
	\begin{proof}
		See Section 1 in the supplementary material.
	\end{proof}
		\vspace{-3mm}
	
	In SGD, the direction of the gradient vector is more important than the length. Theorem \ref{thm:sgd} shows that we can use the first term in Eq. \eqref{eq:grad-sgd} to estimate the descent direction over free variables. The same discussion can be applied to $\left.\bm{v}^{(m)}\right|_{\mathcal{F}_m}$, therefore we can use Algorithm \ref{algo:learning} to learn Eq. \eqref{eq:efm-train-opt}, where $\rho_d$ and $\rho_f$ are the learning rates of dependent variables and free variables, respectively.
	
	Now we discuss the complexity of Algorithm \ref{algo:learning}. Note that in each step, only active factors are used and updated, which means the space complexity follows Proposition \ref{prop:space-complexity}. Moreover, we do not need to compute $\mathcal{B}_{1,p}$ seperately since the it is a part of $\mathcal{B}_{1,m}$. Therefore, the time complexity also follows Corollary \ref{corollary:time}. Therefore, as discussed in Section \ref{section:complexity-discussion}, training RaFM via Algorithm \ref{algo:learning} is as efficient as, or more efficient than training a single FM.
	
	Moreover, it can be proven that the proposed learning algorithm minimizes an upper bound of each FM with a single rank $D_k,1\leq k\leq m$. We put the proof in Section 2 of the supplementary material due to space constraints.

	\section{Related Work} \label{section:related}
	The most related researches on the combination of FMs with different ranks are DiFacto\cite{li2016difacto}, MRMA\cite{li2017mixture}, and FFM\cite{juan2016field}. We compare them with RaFM thoroughly in this section. There are many other researches to improve the performance of FMs by deep models, such as NFM\cite{he2017neural}, AFM\cite{xiao2017attentional}, DeepFM\cite{guo2017deepfm}, etc. The idea of RaFM is orthogonal to these researches. It is also an attractive direction to combine RaFM with these models, and we leave it for future work.
	\subsection{DiFacto}
		DiFacto also uses multiple ranks in one FM. However, for each feature, DiFacto only allocates an embedding vector with a single rank, while RaFM allocates multiple embeddings with different ranks. In DiFacto, the pairwise interaction between features with different ranks is obtained by simply truncating the embedding with a higher rank to a lower rank.
		In cases such as SVD of a complete matrix, such truncations are reasonable, since the best $k$-rank approximation is equivalent to the $k$-prefix of any $k+n$-rank approximation. However, in recommender systems where the training samples are highly sparse,
		such truncations usually lead to worse performances, as will be shown in Section \ref{section:experiment}. Therefore, DiFacto reduces the computational burden of FM by sacrificing the performance, while RaFM improves the performance of FM with a lower computational burden.

	\subsection{MRMA}
	MRMA is a matrix approximation model, of which the key idea is also to combine models with different ranks. The major difference is that it stores the entire models with different ranks, thus leading to a large computational burden and storage burden. Moreover, the large number of parameters may also cause severe overfitting problems, and this is why \cite{li2017mixture} provides the Iterated Conditional Modes (ICM) to train MRMA. In contrast, RaFM will significantly reduce the number of parameters by eliminating inactive factors, and will not be likely to suffer from overfitting problems.
	
	\subsection{FFM}
	RaFM is totally different from FFM, although they both use multiple embeddings for each feature. On the one hand, FFM uses different embeddings for interactions between different field-pairs, while the concept ``field'' never exists in RaFM. For problems without the concept of fields or problems with only 2 fields, FFM fails or degenerates to the original FM, while RaFM still works. On the other hand, different embeddings in FFM are independent, while embeddings in a lower rank can be regarded as the projection of embeddings in higher rank, which is guaranteed by the learning algorithm. This property largely reduces the computational burden and avoids the overfitting problem, while FFM suffers from the large computational burden.
	
	\begin{figure}
		\centering
		\includegraphics[width=0.9\columnwidth]{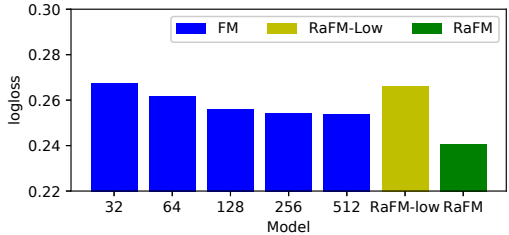}
				\vspace{-3mm}
		\caption{FM vs. RaFM in ML Tag}
		\vspace{-3mm}
		\label{fig:fm-vs-efm}
	\end{figure}
	
	\begin{table}
		\centering
		\caption{Performance and Complexity under Different Settings}
		\begin{tabular}{cccc}
			\hline
			&logloss&\#param & train time \\
			\hline
			FM($D$=512)&0.2538&46.40M&$1\times$ \\
			\hline
			RaFM($S_1$)&0.2405&17.89M&$0.43\times$ \\
			\hline
			RaFM($S_2$)&0.2416&9.63M&$0.17\times$ \\
			\hline
			RaFM($S_3$)&0.2387&9.23M&$0.24\times$ \\
			\hline
			RaFM($S_4$)&0.2391&17.75M&$1.35\times$\\	
			\hline
		\end{tabular}
		\label{tab:complexity-comparison}
	\end{table}
	
	\vspace{-2mm}
	\section{Experiments} \label{section:experiment}
	This section provides the experiments of the proposed approach and other benchmarks in several datasets. In Experiment A, we test our approach on 7 public datasets while in Experiment B, we perform the RaFM approach on the news CTR data provided by Tencent to show the effectiveness of the proposed approach in industrial applications. The code of RaFM is available at \url{https://github.com/cxsmarkchan/RaFM}.
	\begin{table*}[t]
		\centering
		\caption{Results on Regression Tasks}
		\begin{tabular}{cccccccccc}
			\hline
			\hline
			&\multicolumn{3}{c}{ML 10M}&\multicolumn{3}{c}{ML 20M}&\multicolumn{3}{c}{AMovie} \\
			\cline{2-10}
			&square&\multirow{2}{*}{\#param}&train/test&square&\multirow{2}{*}{\#param}&train/test&square&\multirow{2}{*}{\#param}&train/test\\
			&loss&&time&loss&&time&loss&&time\\
			\hline
			\multirow{2}{*}{FM}&0.8016&\multirow{2}{*}{2.66M}&\multirow{2}{*}{$1\times$}&0.8002&\multirow{2}{*}{5.45M}&\multirow{2}{*}{$1\times$}&1.0203&\multirow{2}{*}{3.25M}&\multirow{2}{*}{$1\times$} \\
			&$\pm$0.0010&&&$\pm$0.0008&&&$\pm$0.0046&&\\
			\multirow{2}{*}{DiFacto}&0.7950&\multirow{2}{*}{1.79M}&$0.82\times$/&0.7948&\multirow{2}{*}{3.22M}&$0.70\times$/&1.0268&\multirow{2}{*}{1.76M}&$0.75\times$/\\
			&$\pm$0.0011&&$0.95\times$&$\pm$0.0005&&$0.80\times$&$\pm$0.0051&&$0.75\times$\\
			\multirow{2}{*}{MRMA}&0.7952&\multirow{2}{*}{4.11M}&$1.27\times$/&0.7855&\multirow{2}{*}{8.43M}&$1.19\times$/&1.0071&\multirow{2}{*}{5.02M}&$1.27\times$/\\
			&$\pm$0.0006&&$1.43\times$&$\pm$0.0011&&$1.38\times$&$\pm$0.0039&&$1.27\times$\\
			\multirow{2}{*}{RaFM}&\textbf{0.7870}&\multirow{2}{*}{1.57M}&$0.95\times$/&\textbf{0.7807}&\multirow{2}{*}{3.63M}&$0.74\times$/&\textbf{0.9986}&\multirow{2}{*}{1.76M}&$0.75\times$/\\
			&\textbf{$\pm$0.0008}&&$1.12\times$&\textbf{$\pm$0.0009}&&$0.85\times$&\textbf{$\pm$0.0035}&&$0.75\times$\\
			\hline
			\hline
		\end{tabular}
		\label{tab:result-regression}
	\end{table*}

	\begin{table*}[t]
		\centering
		\caption{Results on Classification Tasks}
		\begin{tabular}{ccccccccc}
			\hline
			\hline
			&\multicolumn{4}{c}{Frappe}&\multicolumn{4}{c}{ML Tag} \\
			\cline{2-9}
			&log loss&AUC&\#param&train/test time&log loss&AUC&\#param&train/test time\\
			\hline
			\multirow{2}{*}{FM}&0.1702&0.9771&\multirow{2}{*}{1.38M}&\multirow{2}{*}{$1\times$}&0.2538&0.9503&\multirow{2}{*}{46.40M}&\multirow{2}{*}{$1\times$} \\
			&$\pm$0.0023&$\pm$0.0008&&&$\pm$0.0009&$\pm$0.0006&& \\
			\multirow{2}{*}{DiFacto}&0.1711&0.9771&\multirow{2}{*}{0.61M}&$0.63\times$/&0.2529&0.9450&\multirow{2}{*}{16.97M}&$0.42\times$/\\
			&$\pm$0.0023&$\pm$0.0004&&$0.85\times$&$\pm$0.0007&$\pm$0.004&&$0.83\times$ \\
			\multirow{2}{*}{RaFM}&\textbf{0.1447}&\textbf{0.9811}&\multirow{2}{*}{0.71M}&$0.73\times$/&\textbf{0.2387}&\textbf{0.9526}&\multirow{2}{*}{9.23M}&$0.24\times$/\\
			&\textbf{$\pm$0.0015}&\textbf{$\pm$0.0002}&&$0.85\times$&\textbf{$\pm$0.0005}&\textbf{$\pm$0.0006}&&$0.71\times$ \\
			\hline
			\hline
			&\multicolumn{4}{c}{Avazu}&\multicolumn{4}{c}{Criteo} \\
			\cline{2-9}
			&log loss&AUC&\#param&train/test time&log loss&AUC&\#param&train/test time\\
			\hline
			\multirow{2}{*}{FM}&0.3817&0.7761&\multirow{2}{*}{18.81M}&\multirow{2}{*}{$1\times$}&0.4471&0.8030&\multirow{2}{*}{35.87M}&\multirow{2}{*}{$1\times$} \\
			&$\pm$0.0001&$\pm$0.0003&&&$\pm$0.0002&$\pm$0.0002&& \\
			\multirow{2}{*}{DiFacto}&0.3823&0.7778&\multirow{2}{*}{10.83M}&$0.82\times$/&0.4470&0.8030&\multirow{2}{*}{19.70M}&$0.63\times$/\\
			&$\pm$0.0003&$\pm$0.0003&&$1.79\times$&$\pm$0.0002&$\pm$0.0004&&$0.80\times$ \\
			\multirow{2}{*}{RaFM}&\textbf{0.3801}&\textbf{0.7826}&\multirow{2}{*}{10.17M}&$0.85\times$/&\textbf{0.4451}&\textbf{0.8060}&\multirow{2}{*}{20.88M}&$0.67\times$/\\
			&\textbf{$\pm$0.0002}&\textbf{$\pm$0.0003}&&$1.20\times$&\textbf{$\pm$0.0001}&\textbf{$\pm$0.0002}&&$0.84\times$ \\
			\hline
			\hline
		\end{tabular}
		\label{tab:result-classification}
				\vspace{-4mm}
	\end{table*}
	
	\subsection{Experiment A: Public Datasets}
	\subsubsection{Experiment Setup}
	We consider 3 datasets for regression tasks and 4 for classification tasks.	All these datasets are randomly split into train (80\%), validation (10\%), and test (10\%) sets. Datasets for regression tasks are the MovieLens 10M (ML 10M), 20M (ML 20M), and the Amazon movie review dataset\footnote{\url{http://jmcauley.ucsd.edu/data/amazon/}} (AMovie), respectively, of which the square loss is used as the performance criterion. Datasets for classification tasks are Frappe\footnote{\url{http://baltrunas.info/research-menu/frappe}}, Movielens Tag (ML Tag), Avazu, and Criteo\footnote{\url{http://labs.criteo.com/2014/02/kaggle-display-advertising-challenge-dataset/}}, respectively, of which the log loss and the area under curve (AUC) are used as the performance criteria.
	
	We use the standard FM \cite{rendle2010factorization} and Difacto \cite{li2016difacto} as baselines for all datasets, and MRMA as a baseline for datasets containing exactly two fields, i.e. ML 10M, ML 20M, and AMovie, respectively. We do not compare with deep models since
	the motivation of this work is to show that FMs with different ranks can be efficiently combined and trained while incurring less computational burden. However, we argue that RaFM is a flexible framework and 
	deep FMs can be embedded. We adopt $L_2$ regularizations for each model, and search the $L_2$ coefficient from $\{1e^{-6},5e^{-6},1e^{-5},\dots,1e^{-1}\}$ on the validation set.
	We search the ranks from $\{32,64,128,256,512\}$ for each model, except for some discussions in Section \ref{section:experiment-complexity}, where we use more values.
	We compare the performance, the model size, the training time, and the test time. Since RaFM has more hyperparameters, to be fair, we tune $k_i$ by the following equation rather than grid search:
	\begin{equation}
	k_i = \arg\min_k \left|\log n_i -\log D_k\right|
	\end{equation}
	where $n_i$ is the number of occurrences of the $i$-th feature. This equation means choosing $k_i$ so that $D_k$ is the closest to $n_i$ in the sense of logarithm.
	Moreover, we use the same $L_2$ coefficient for all FM models in the RaFM, and search it from the same candidate set as baselines.
	
	\subsubsection{Advantages of RaFM over FM with Fixed Ranks}
	Fig. \ref{fig:fm-vs-efm} shows the comparison between FMs with different ranks and RaFM in the ML Tag dataset. The ranks of FMs range from 32 to 512, and the hyperparameters of RaFM are $m=2, D_1=32$ and $D_2=512$. The RaFM-low in Fig. \ref{fig:fm-vs-efm} represents the $\mathcal{B}_{1,1}$, i.e. the FM model with rank 32, but trained simultaneously with $\mathcal{B}_{1,2}$ by Eq. \eqref{eq:efm-train-opt-constraint}. It is shown that RaFM significantly outperforms all FMs. It is because that RaFM maintains 32 factors for most of the features with limited occurrences to avoid overfitting problems, and allocate 512 factors for features with enough occurrences to guarantee expressiveness. Moreover, RaFM-low achieves similar performance to FM with 32 factors, which means embeddings in RaFM has similar expressiveness of embeddings in FM with the same rank.
	
	\subsubsection{Performance and Complexity} \label{section:experiment-complexity}
	Table \ref{tab:complexity-comparison} shows the performance and complexity of RaFMs under different rank settings. We compare 4 sets of ranks, namely, $S_1 = \{32,512\}$, $S_2=\{32,128,512\}$, $S_3=\{32,64,128,256,512\}$, and $S_4=\{32,64,96,128,160,\cdots,480,512\}$. Therefore, ranks in $S_2,S_3$ varies exponentially, while ranks in $S_4$ varies arithmetically. RaFMs with these 4 sets achieves similar performances, all significantly better than FM with 512 factors, which is the best out of all FMs as shown in Fig. \ref{fig:fm-vs-efm}. RaFMs with $S_2,S_3$ have fewer parameters due to the expenentially increasing rank settings. RaFM with $S_4$ has a comparable number of parameters to $S_1$, but the train time is much longer. This comparison shows the impacts of the speed of rank increasing on the computational burden, which is consistent with what is discussed in Section \ref{section:complexity-discussion}.
	
	\subsubsection{Results and Discussion}
	As evidenced by Tables \ref{tab:result-regression} and \ref{tab:result-classification}, our algorithm significantly outperforms the baseline algorithms. For example, the relative improvements on square loss (log loss) criteria are $1\%\sim2\%$ in ML 10M, ML 20M 
	and AMovie, $15\%$ on Frappe, $6\%$ on ML Tag, and 0.5\% on Avazu and Criteo. Empirically, all these improvements are regarded as significant in the researches 
	on corresponding datasets. Taking Criteo as an example, RaFM reduces the logloss of FM by 0.002, while an improvement of $0.001$ in logloss is considered as practically significant \cite{wang2017deep}.
		
	Moreover, the complexity of RaFM is reduced compared to FMs. In fact, RaFM significantly reduces both the model size and the computational time. For example, the model size of RaFM is only 20\%$\sim$66\% of FM, and the training time is 24\%$\sim$95\% of FM.
	
	
	Although Difacto can also reduce the computational burden, it cannot guarantee the performance. The major reason is that Difacto assumes the low-dimensional counterpart of a high-dimensional feature can be obtained by parameter sharing. However, this assumption is not usually reasonable. In contrast, MRMA performs better than FM and DiFacto due to it combines models with different ranks. However, its model size and training time are significantly larger than RaFM, while its performance is a bit worse than that of RaFM due the overfitting problem caused by a large number of parameters.
	
	In summary, RaFM not only achieves a better performance, but also reduces the model size and computational time. Therefore, the proposed RaFM model also has an attractive potential in industrial applications.
	
	\subsection{Experiment B: Industrial Level Click-Through-Rate Dataset}
	\subsubsection{Experiment Setup}
	Here we perform an experiment on the news CTR data provided by Tencent in order to show the potential of RaFM in industrial applications. We use the records of 8 consecutive days, of which the first 7 days are used as the training dataset, the last day is used on the test dataset. The dataset contains 1.7 billion records and 120 millions features.
	
	We compare the proposed RaFM approach with LR, FM and DiFacto, and the performance standard is the AUC, which is consistent with the requirement of online predictions. For FM/DiFacto/RaFM, the dimensionality of factors are chosen from $\{1,2,4,8\}$. We use the FTRL algorithm to guarantee the sparsity of learned weights, and use distributed learning to accelerate the learning process.
	
	\subsubsection{Results}
	We show the comparisons of AUC and the model size in Fig.~\ref{fig:kuaibao}. We use the number of parameters of LR, which is 13.19M in our experiments, as the unit of the model size. The comparison of training time of these approaches is not provided here, because they are very similar (about 6 hours for training, and 30 minutes for testing in our experiments) due to the fact that the training timeis dominated by the communication time of distributed learning rather than the computational time.
	\begin{figure}
		\centering
		\includegraphics[width=\columnwidth]{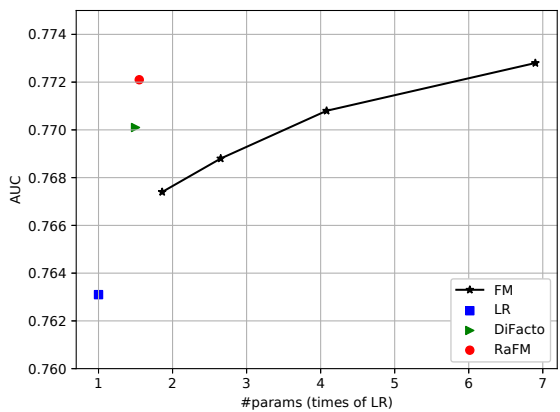}
		\caption{Comparisons of AUC and the model size on industrial level CTR dataset.}
		\label{fig:kuaibao}
	\end{figure}
	
	Since the number of records in the training dataset is sufficiently large, it is natural to hope that the AUC will continuously increase when the model size increases, and the AUCs in Fig.~\ref{fig:kuaibao} can be further improved if we allow a larger rank of factors. However, the model size to achieve such an improvement is the main issue, since a larger model size will undoubtedly bring a heavier burden to the parameter server. According to Fig.~\ref{fig:kuaibao}, RaFM increases the AUC by about 1\% compared to LR, while its model size is only 1.55 times that of LR. In comparison, FM needs 7 times the model size of LR to achieve a similar performance. In other words, the model size of RaFM is only 22\% that of FM to achieve a similar AUC, which means the proposed RaFM approach achieves a good trade-off between model size and performance, and has an attractive potential in industrial applications.

	\section{Conclusion and Future Work} \label{section:conclusion}
	This paper proposes an RaFM model which adopts pairwise interactions from embeddings with different ranks. RaFM can be stored and evaluated as efficiently as, or even more efficiently than FMs with fixed ranks. Moreover, we provide a learning algorithm for efficiently training all embeddings in one concise model, and prove that the training error of each FM is bounded from above.
	Experiments demonstrate that RaFM not only has better performance in regression and classification datasets whose different features have significantly varying frequencies of occurrence, but also reduces the computational burden of FMs.
	
	RaFM is a flexible framework, therefore an interesting direction in future study is to combine it with deep models in order to achieve better performances. Moreover, a more effective hyperparameter tuning approach is also an attractive research direction.
	

	\bibliographystyle{icml2019}
	\bibliography{efm}
\end{document}


	\onecolumn
	\begin{center}
	{\LARGE Supplementary Materials for
	
	``RaFM: Rank-Aware Factorization Machines''}
	\end{center}
	
	\section{Proof of Theorem 6}
		\begin{proof}
		Recall the estimated gradient:
		\begin{equation*}
		\begin{split}
		\widehat{grad} = \frac{\partial L(\mathcal{B}_{1,m},y)}{\partial \left.\bm{v}^{(p)}\right|_{\mathcal{F}_p-\mathcal{F}_{p+1}}} +\frac{1}{N}\sum_{\bm{x'}} \frac{\partial L(\mathcal{B}'_{1,m},y)}{\partial \left.\bm{v}^{(p-1)}\right|_{\mathcal{F}_p}}
		\bm{G}^{-1}
		\frac{\partial^2 L(\mathcal{B}_{1,p-1},\mathcal{B}_{1,p})}{\partial \left.\bm{v}^{(p-1)}\right|_{\mathcal{F}_p} \partial \left.\bm{v}^{(p)}\right|_{\mathcal{F}_p-\mathcal{F}_{p+1}}}
		\end{split}
		\end{equation*}
		According to the chain rule and the fact that $\mathcal{B}_{1,p-1}$ does not contain $\left.\bm{v}^{(p)}\right|_{\mathcal{F}_p-\mathcal{F}_{p+1}}$, we have
		\begin{equation} \label{eq:2nd-derivative}
		\begin{split}
		&\frac{\partial^2 L(\mathcal{B}_{1,p-1},\mathcal{B}_{1,p})}{\partial \left.\bm{v}^{(p-1)}\right|_{\mathcal{F}_p} \partial \left.\bm{v}^{(p)}\right|_{\mathcal{F}_p-\mathcal{F}_{p+1}}} \\
		=& L_{12}(\mathcal{B}_{1,p-1},\mathcal{B}_{1,p}) \left(\frac{\partial \mathcal{B}_{1,p-1}}{\partial \left.\bm{v}^{(p-1)}\right|_{\mathcal{F}_p}}\right)^{\top}\frac{\partial \mathcal{B}_{1,p}}{\partial \left.\bm{v}^{(p)}\right|_{\mathcal{F}_p-\mathcal{F}_{p+1}}} \\
		&+L_{22}(\mathcal{B}_{1,p-1},\mathcal{B}_{1,p}) \left(\frac{\partial \mathcal{B}_{1,p}}{\partial \left.\bm{v}^{(p-1)}\right|_{\mathcal{F}_p}}\right)^{\top}\frac{\partial \mathcal{B}_{1,p}}{\partial \left.\bm{v}^{(p)}\right|_{\mathcal{F}_p-\mathcal{F}_{p+1}}} \\
		&+L_{2}(\mathcal{B}_{1,p-1},\mathcal{B}_{1,p}) \frac{\partial^2 \mathcal{B}_{1,p}}{\partial \left.\bm{v}^{(p-1)}\right|_{\mathcal{F}_p}\partial\left.\bm{v}^{(p)}\right|_{\mathcal{F}_p-\mathcal{F}_{p+1}}}
		\end{split}
		\end{equation}
		where $L_2$ means the partial derivative of $L$ with regard to its first value, while $L_{12}$ and $L_{22}$ are the second-order partial derivatives of $L$. Note that the last line of \eqref{eq:2nd-derivative} equals 0 because that each pairwise interaction in the RaFM will not contain vectors from different FMs. Therefore, 
		\begin{equation*}
		\begin{split}
		\frac{\partial^2 L(\mathcal{B}_{1,p-1},\mathcal{B}_{1,p})}{\partial \left.\bm{v}^{(p-1)}\right|_{\mathcal{F}_p} \partial \left.\bm{v}^{(p)}\right|_{\mathcal{F}_p-\mathcal{F}_{p+1}}} 
		= \bm{H}\frac{\partial \mathcal{B}_{1,p}}{\left.\bm{v}^{(p)}\right|_{\mathcal{F}_p-\mathcal{F}_{p+1}}}
		\end{split}
		\end{equation*}
		where $\bm{H}$ is a $(\left|\mathcal{F}_p\right|D_{p-1})\times 1$ matrix:
		\begin{equation*}
		\bm{H}=L_{12}(\mathcal{B}_{1,p-1},\mathcal{B}_{1,p})\left(\frac{\partial \mathcal{B}_{1,p-1}}{\left.\bm{v}^{(p-1)}\right|_{\mathcal{F}_p}}\right)^{\top} + L_{22}(\mathcal{B}_{1,p-1},\mathcal{B}_{1,p})\left(\frac{\partial \mathcal{B}_{1,p}}{\left.\bm{v}^{(p-1)}\right|_{\mathcal{F}_p}}\right)^{\top}
		\end{equation*}
		Then we have
		\begin{equation*}
		\begin{split}
		\widehat{grad} =& L_1(\mathcal{B}_{1,m},y)\frac{\partial \mathcal{B}_{1,m}}{\partial \left.\bm{v}^{(p)}\right|_{\mathcal{F}_p-\mathcal{F}_{p+1}}} \\
		&+\frac{1}{N}\sum_{\bm{x'}} L_1(\mathcal{B}'_{1,m},y)\frac{\partial \mathcal{B}'_{1,m}}{\partial \left.\bm{v}^{(p-1)}\right|_{\mathcal{F}_p}}
		\bm{G}^{-1}\bm{H}\frac{\partial \mathcal{B}_{1,p}}{\left.\bm{v}^{(p)}\right|_{\mathcal{F}_p-\mathcal{F}_{p+1}}}\\
		=& L_1(\mathcal{B}_{1,m},y)\frac{\partial \mathcal{B}_{1,m}}{\partial \left.\bm{v}^{(p)}\right|_{\mathcal{F}_p-\mathcal{F}_{p+1}}} + \lambda\frac{\partial \mathcal{B}_{1,p}}{\left.\bm{v}^{(p)}\right|_{\mathcal{F}_p-\mathcal{F}_{p+1}}}
		\end{split}
		\end{equation*}
		where
		\begin{equation*}
		\lambda = \frac{1}{N}\sum_{\bm{x'}} L_1(\mathcal{B}'_{1,m},y)\frac{\partial \mathcal{B}'_{1,m}}{\partial \left.\bm{v}^{(p-1)}\right|_{\mathcal{F}_p}}
		\bm{G}^{-1}\bm{H}
		\end{equation*}
		Note that $\lambda$ is the multiplication of a $1\times \left|\mathcal{F}_p\right|D_{p-1}$ matrix, a $\left|\mathcal{F}_p\right|D_{p-1}\times \left|\mathcal{F}_p\right|D_{p-1}$ matrix, and a $\left|\mathcal{F}_p\right|D_{p-1}\times 1$ matrix, and thus is a scalar. Moreover, the derivative of $\mathcal{B}_{1,m}$ and $\mathcal{B}_{1,p}$ with respect to $\left.\bm{v}^{(p)}\right|_{\mathcal{F}_p-\mathcal{F}_{p+1}}$ is the same. Therefore, the direction of $\widehat{grad}$ is parallel to that of $L_1(\mathcal{B}_{1,m},y)\left(\partial \mathcal{B}_{1,m}/\partial \left.\bm{v}^{(p)}\right|_{\mathcal{F}_p-\mathcal{F}_{p+1}}\right)$, i.e. $\partial L(\mathcal{B}_{1,m},y)/\partial \left.\bm{v}^{(p)}\right|_{\mathcal{F}_p-\mathcal{F}_{p+1}}$.
	\end{proof}

	\section{Performance Bound of the Learning Algorithm}
	\begin{theorem} \label{thm:bound}
		Assume there exist two nonnegative functions $d(\cdot)$ and $\Delta(\cdot,\cdot)$ such that $d$ is monotonically increasing, and for all $f_1(\cdot),f_2(\cdot)$ we have
		\begin{equation} \label{eq:triangle}
		\begin{split}
		d\left(\frac{1}{N}\sum_{\bm{x}}L(f_1(\bm{x}),y)\right) &\leq d\left(\frac{1}{N}\sum_{\bm{x}}L(f_2(\bm{x}),y)\right) \\
		&+ d\left(\frac{1}{N}\sum_{\bm{x}}\Delta(f_1(\bm{x}),f_2(\bm{x}))\right)
		\end{split}
		\end{equation}
		then regarding the training error of the $k$-th FM model, i.e. $\mathcal{B}_{k,k}$, we have
		\begin{equation}
		\begin{split}
		d\left(\frac{1}{N}\sum_{\bm{x}}L(\mathcal{B}_{k,k}^*,y)\right)&\leq d\left(\frac{1}{N}\sum_{\bm{x}}L(\mathcal{B}_{1,m}^*,y)\right)\\
		&+ \sum_{p=k}^{m-1}d\left(\frac{1}{N}\sum_{\bm{x}}\Delta(\mathcal{B}_{1,p}^*,\mathcal{B}_{1,p+1}^*)\right) \label{eq:upper-bound}
		\end{split}
		\end{equation}
		where $\mathcal{B}_{1,m}^*$ is the optimal $\mathcal{B}_{1,m}$, and $\mathcal{B}_{1,p}^*$ is defined in the same way.
	\end{theorem}
	
	\textbf{Remark}:
	In practice, $\Delta$ represents the error of expressing $f_2$ by $f_1$. Eq. \eqref{eq:triangle} is an extension to the triangle inequality, and can be applied to both regression tasks and classification tasks. In regression tasks, $L$ is the square loss, then we can set $d$ as the square root function and $\Delta=L$. In classification tasks, $L$ is the logarithm loss, then we can let $d$ be an identity function, and define $\Delta$ as
	\begin{equation}
	\Delta(f_1(\bm{x}),f_2(\bm{x})) = C_{\theta,\delta}\mathcal{D}_{KL}\left[f_1(\bm{x})\|f_2(\bm{x})\right] + \log \delta
	\end{equation}
	where $\delta>1$, $C_{\theta,\delta} = \frac{\log \delta}{\theta\log\delta + (1-\theta)\log\frac{1-\theta}{1-\theta/\delta}}$, and $\theta=\min_{\bm{x}}\left[yf_2(\bm{x})+ (1-y)(1-f_2(\bm{x}))\right]$. $\mathcal{D}_{KL}$ is the KL divergence of two bimonial variables. The readers can refer to Section 3 in the supplementary material for the proof of Eq.  \eqref{eq:triangle} for logarithm loss.
	
	In order to prove Theorem \ref{thm:bound}, we first provide the following lemma:
	\begin{lemma}
		The following inequalities hold
		\begin{equation} \label{eq:model-capacity}
	  d\left(\frac{1}{N}\sum_{\bm{x}}l(\mathcal{B}_{l,k}^*,y)\right) \leq d\left(\frac{1}{N}\sum_{\bm{x}}L(\mathcal{B}_{l-1,k}^*,y)\right)
		\end{equation}
	\end{lemma}
	\begin{proof}
	Note that we have $\mathcal{B}_{l,k}=\mathcal{B}_{l-1,k}$ provided that
	\begin{equation*}
	\bm{v}_i^{(l)} = 
	\begin{bmatrix}
	\bm{0}_{D_{l}-D_{l-1}} \\
	\bm{v}_i^{(l-1)}
	\end{bmatrix}
	, \forall i \in \mathcal{F}_{l}
	\end{equation*}
	And such solution also satisfies the constraint (12) in the main body of the paper. Therefore $\mathcal{B}_{l-1,k}$ is a submodel of $\mathcal{B}_{l,k}$, and thus the optimal training error of $\mathcal{B}_{l,k}$ is smaller than $\mathcal{B}_{l-1,k}$, and \eqref{eq:model-capacity} follows.
	\end{proof}

	\begin{proof}[Proof of Theorem \ref{thm:bound}]
		According to \eqref{eq:model-capacity} we have
		\begin{equation*}
		\begin{split}
		d\left(\frac{1}{N}\sum_{\bm{x}}L(\mathcal{B}_{k,k}^*,y)\right)&\leq d\left(\frac{1}{N}\sum_{\bm{x}}L(\mathcal{B}_{k-1,k}^*,y)\right)\leq d\left(\frac{1}{N}\sum_{\bm{x}}L(\mathcal{B}_{k-2,k}^*,y)\right)\\
		&\dots\\
		&\leq d\left(\frac{1}{N}\sum_{\bm{x}}L(\mathcal{B}_{1,k}^*,y)\right)
		\end{split}
		\end{equation*}
		Moreover, according to \eqref{eq:triangle} we have
		\begin{equation*}
		\begin{split}
		d\left(\frac{1}{N}\sum_{\bm{x}}L(\mathcal{B}_{1,k}^*,y)\right)&\leq d\left(\frac{1}{N}\sum_{\bm{x}}L(\mathcal{B}_{1,k+1}^*,y)\right) + d\left(\frac{1}{N}\sum_{\bm{x}}\Delta(\mathcal{B}_{1,k}^*,\mathcal{B}_{1,k+1}^*)\right)\\
		&\leq d\left(\frac{1}{N}\sum_{\bm{x}}L(\mathcal{B}_{1,k+2}^*,y)\right) + \sum_{p=k}^{k+1}d\left(\frac{1}{N}\sum_{\bm{x}}\Delta(\mathcal{B}_{1,p}^*,\mathcal{B}_{1,p+1}^*)\right)\\
		&\dots\\
		&\leq d\left(\frac{1}{N}\sum_{\bm{x}}L(\mathcal{B}_{1,m}^*,y)\right) + \sum_{p=k}^{m-1}d\left(\frac{1}{N}\sum_{\bm{x}}\Delta(\mathcal{B}_{1,p}^*,\mathcal{B}_{1,p+1}^*)\right)
		\end{split}
		\end{equation*}
		
		Therefore we have
		\begin{equation*}
		d\left(\frac{1}{N}\sum_{\bm{x}}L(\mathcal{B}_{k,k}^*,y)\right)\leq d\left(\frac{1}{N}\sum_{\bm{x}}L(\mathcal{B}_{1,m}^*,y)\right) + \sum_{p=k}^{m-1}d\left(\frac{1}{N}\sum_{\bm{x}}\Delta(\mathcal{B}_{1,p}^*,\mathcal{B}_{1,p+1}^*)\right)
		\end{equation*}
		
	\end{proof}

	\section{Quasi-Triangle Inequality for Logarithmic Loss}
	The following proposition is an extension of the triangle inequality for log loss.
	\begin{proposition} \label{prop:triangle-logloss}
		Suppose $y\in\{0,1\}, 0< \hat{y}_1,\hat{y}_2< 1$, and define the log loss function $L(\hat{y}_i,y)$ and the KL divergence $\mathcal{D}_{KL}(\hat{y}_1\|\hat{y}_2)$ as
		\begin{equation*}
		\begin{split}
		L(\hat{y}_i,y)&=-y\log\hat{y}_i-(1-y)\log(1-\hat{y}_i) \\
		\mathcal{D}_{KL}(\hat{y}_1\|\hat{y}_2)&=\hat{y}_1\log\frac{\hat{y}_1}{\hat{y}_2}+(1-\hat{y}_1)\log\frac{1-\hat{y}_1}{1-\hat{y}_2}
		\end{split}
		\end{equation*}
		then $\forall \delta > 1,0<\theta\leq y\hat{y}_1 + (1-y)(1-\hat{y}_1)$, we have
		\begin{equation}
		L(\hat{y}_2,y)\leq L(\hat{y}_1,y)+C_{\theta,\delta}\mathcal{D}_{KL}(\hat{y}_1\|\hat{y}_2)+\log\delta
		\end{equation}
		where
		\begin{equation}
		C_{\theta,\delta} = \frac{\log \delta}{\theta\log\delta + (1-\theta)\log\frac{1-\theta}{1-\theta/\delta}}
		\end{equation}
	\end{proposition}
	
	Before proving Proposition \ref{prop:triangle-logloss}, we first provide some lemmas.
	\begin{lemma} \label{lemma:monotonicity}
		$\forall \theta>0,\delta>1$, we have $C_{\theta,\delta}>0$, and $C_{\theta,\delta}$ monotonically decreases when $\theta$ increases.
	\end{lemma}
	\begin{proof}
		Consider $g(\theta)=1/C_{\theta,\delta}=\theta + \left[(1-\theta)\log\frac{1-\theta}{1-\theta/\delta}/\log\delta\right]$, then we have
		\begin{equation*}
		\begin{split}
		g'(\theta)\log\delta&=\log\delta - \log\frac{1-\theta}{1-\theta/\delta} - 1 + \frac{1}{\delta}\frac{1-\theta}{1-\theta/\delta} \\
		&=- \log\left(\frac{1}{\delta}\frac{1-\theta}{1-\theta/\delta}\right) +\left( \frac{1}{\delta}\frac{1-\theta}{1-\theta/\delta}-1\right) > 0
		\end{split}
		\end{equation*}
		here we use the fact that $\log x< (x-1)$ unless $x=1$. Due to that $\log\delta>0$, we have $g'(\theta)\geq 0$, thus $g(\theta)$ is an increasing function. Moreover we have
		\begin{equation*}
		g(\theta)> g(0) = 0
		\end{equation*}
		Therefore, $C_{\theta,\delta}=1/g(\theta)$ is a decreasing function  with respect to $\theta$, and $C_{\theta,\delta}>0$.
	\end{proof}
	\begin{lemma} \label{lemma:simplified}
		For $0<\hat{y}_1,\hat{y}_2<1$, we have
		\begin{equation} \label{eq:triangle-simplified}
		\log \frac{\hat{y}_1}{\hat{y}_2} \leq C_{\hat{y}_1,\delta}\mathcal{D}_{KL}(\hat{y}_1\|\hat{y}_2) + \log \delta
		\end{equation}
	\end{lemma}
	\begin{proof}
		When $\hat{y}_1<\delta\hat{y}_2$, we have $\log (\hat{y}_1/\hat{y}_2) \leq \log \delta$, thus \eqref{eq:triangle-simplified} holds due to the nonnegativity of the KL divergence. Now we discuss the case when $\hat{y}_1\geq \delta\hat{y}_2$. Consider the ratio between $\mathcal{D}_{KL}(\hat{y}_1\|\hat{y}_2)$ and $\log (\hat{y}_1/\hat{y}_2)$:
		\begin{equation*}
		\frac{\mathcal{D}_{KL}(\hat{y}_1\|\hat{y}_2)}{\log (\hat{y}_1/\hat{y}_2)}=\hat{y}_1 + (1-\hat{y}_1)\frac{\log(1-\hat{y}_1) -\log(1-\hat{y}_2)}{\log \hat{y}_1-\log \hat{y}_2}=\hat{y}_1 + (1-\hat{y}_2)h(\hat{y}_1,\hat{y}_2)
		\end{equation*}
		where $h(\hat{y}_1,\hat{y}_2) = \log\frac{1-\hat{y}_1}{1-\hat{y}_2}/\log\frac{\hat{y}_1}{\hat{y}_2}$. It is easy to show that
		\begin{equation*}
		\frac{\partial h}{\partial \hat{y}_2} = -\frac{\mathcal{D}_{KL}(\hat{y}_2\|\hat{y}_1)}{(\log\hat{y}_1-\log\hat{y}_2)^2\hat{y}_2(1-\hat{y}_2)} \leq 0
		\end{equation*}
		Therefore according to $\hat{y}_1\geq \delta \hat{y}_2$, we have $h(\hat{y}_1,\hat{y}_2)\geq h(\hat{y}_1,\hat{y}_1/\delta)$, and
		\begin{equation*}
		\frac{\mathcal{D}_{KL}(\hat{y}_1\|\hat{y}_2)}{\log (\hat{y}_1/\hat{y}_2)} \geq\hat{y}_1 + (1-\hat{y}_1)\frac{\log(1-\hat{y}_1) -\log(1-\hat{y}_1/\delta)}{\log \hat{y}_1-\log (\hat{y}_1/\delta)} = \frac{1}{C_{\hat{y}_1,\delta}}
		\end{equation*}
		Therefore \eqref{eq:triangle-simplified} also holds when $\hat{y}_1\geq\delta\hat{y}_2$.
	\end{proof}
	\begin{proof}[Proof of Proposition \ref{prop:triangle-logloss}]
		We first prove the case when $y=1$. In this case, we have $\theta \leq \hat{y}_1$, and
		\begin{equation*}
		\begin{split}
		L(\hat{y}_2,y)-L(\hat{y}_1,y) &= \log\frac{\hat{y}_1}{\hat{y}_2} \\
		&\leq C_{\hat{y}_1,\delta}\mathcal{D}_{KL}(\hat{y}_1\|\hat{y}_2) + \log \delta \\
		&\leq C_{\theta,\delta}\mathcal{D}_{KL}(\hat{y}_1\|\hat{y}_2) + \log \delta
		\end{split}
		\end{equation*}
		where the first and second inequalities are according to Lemmas \ref{lemma:simplified} and \ref{lemma:monotonicity}. For $y=0$, we can let $y'=1-y,\hat{y}'_1=1-\hat{y}_1,$ and $\hat{y}'_2=1-\hat{y}_2$, and use the same discussion for $y',\hat{y}'_1$ and $\hat{y}'_2$.
	\end{proof}